\documentclass{article}
\usepackage{amsmath, amssymb, amsthm}
\usepackage{bbm}
\usepackage{bm}
\usepackage{graphicx}
\usepackage[colorlinks=true, urlcolor=blue, linkcolor=red, citecolor=green]{hyperref}
\usepackage{booktabs}
\usepackage{xcolor}
\usepackage{microtype}
\usepackage{algorithm}
\usepackage{algpseudocode}

\newtheorem{theorem}{Theorem}
\newtheorem{lemma}[theorem]{Lemma}

\newtheorem{definition}[theorem]{Definition}

\theoremstyle{definition}

\title{On the Complexity of Optimal Graph Rewiring for Oversmoothing and Oversquashing in Graph Neural Networks}

\author{Mostafa Haghir Chehreghani \\
	Department of Computer Engineering \\
	Amirkabir University of Technology (Tehran Polytechnic) \\
	Tehran, Iran \\ mostafa.chehreghani@aut.ac.ir}

\date{}

\begin{document}
	
	\maketitle
	
	\begin{abstract}
		Graph Neural Networks (GNNs) face two fundamental challenges when scaled to deep architectures: oversmoothing, where node representations converge to indistinguishable vectors, and oversquashing, where information from distant nodes fails to propagate through bottlenecks. Both phenomena are intimately tied to the underlying graph structure, raising a natural question: can we optimize the graph topology to mitigate these issues? This paper provides a theoretical investigation of the computational complexity of such graph structure optimization. We formulate oversmoothing and oversquashing mitigation as graph optimization problems based on spectral gap and conductance, respectively. We prove that exact optimization for either problem is NP‑hard through reductions from Minimum Bisection, establishing NP‑completeness of the decision versions. Our results provide theoretical foundations for understanding the fundamental limits of graph rewiring for GNN optimization and justify the use of approximation algorithms and heuristic methods in practice.
	\end{abstract}
	
	
	\section{Introduction}
	
	Graph Neural Networks (GNNs) \cite{kipf2016semi, velivckovic2017graph,DBLP:journals/tjs/ZohrabiSC24,DBLP:journals/tweb/NasrabadiKZC25} have become the de facto standard for learning on graph-structured data, achieving remarkable success in applications ranging from fake news detection \cite{DBLP:journals/kbs/LakzaeiCB25,DBLP:journals/asc/LakzaeiCB25} to
	drug discovery \cite{Huang2024TxGNN} and recommendation systems \cite{DBLP:journals/ijon/GholinejadC26}. Their core operation, message passing, aggregates information from neighboring nodes, allowing the model to capture relational structure. However, as GNNs are scaled to deeper architectures, they suffer from two fundamental limitations that severely restrict their expressive power and performance: oversmoothing and oversquashing.
	
	\textbf{Oversmoothing} occurs when repeated message passing causes node representations to converge to a uniform vector, erasing discriminative information \cite{li2018deeper, oono2019graph,DBLP:journals/kbs/HoseinniaGC25}. This phenomenon arises because GNN layers act as low‑pass filters on the graph; after many layers, the features become dominated by the principal eigenvector of the propagation matrix. The rate of oversmoothing is governed by the spectral gap of the graph’s propagation operator, and it has been shown that for any graph with a fixed spectrum, oversmoothing is inevitable as the number of layers grows \cite{oono2019graph,DBLP:journals/kbs/HoseinniaGC25}. Consequently, many modern GNN architectures incorporate skip connections or normalization to mitigate the effect, but the underlying graph structure remains a key determinant.
	
	\textbf{Oversquashing} describes the failure of information from distant nodes to reach a given node due to bottlenecks in the graph \cite{alon2021bottleneck,DBLP:journals/asc/MohamadiC25}. When messages must flow through a narrow cut, the message‑passing process compresses exponentially many messages into fixed‑size vectors, leading to information loss. Oversquashing is fundamentally a structural phenomenon: it is quantified by the graph’s conductance or effective resistance \cite{to2021understanding, nguyen2023understanding}. A graph with a bottleneck forces information to be squashed, no matter how many layers are used.
	
	Crucially, both oversmoothing and oversquashing are determined by the underlying graph topology, yet they often conflict. Graphs that mitigate oversquashing — such as expanders or dense graphs — tend to have high connectivity, which accelerates oversmoothing. Conversely, sparse, tree‑like graphs resist oversmoothing but create severe bottlenecks that cause oversquashing. This tension raises a natural question: can we optimize the graph structure itself to strike an optimal balance between these two opposing forces?
	
	Recent works have explored graph rewiring as a practical tool to improve GNN performance. Heuristic methods such as DIGL \cite{gasteiger2019diffusion} and SDRF \cite{to2021understanding} iteratively add or remove edges based on diffusion or curvature, demonstrating empirical success. However, the fundamental computational complexity of such graph optimization problems has remained unexplored. Is it possible to efficiently find a graph that optimally mitigates oversmoothing or oversquashing? Or are these problems inherently intractable, justifying the use of approximations and heuristics?
	
	\subsection{Contributions}
	
	This paper provides the first theoretical investigation of the complexity of graph structure optimization for GNN oversmoothing and oversquashing mitigation. Our main contributions are as follows.
	
	\begin{enumerate}
		\item \textbf{Unified Formulations:} We cast the mitigation of oversmoothing and oversquashing as graph optimization problems with natural objective functions. For oversmoothing we use the second eigenvalue of the propagation matrix (equivalently the spectral gap of the normalized Laplacian). For oversquashing we use conductance, which captures the worst‑case bottleneck.
		
		\item \textbf{NP‑Hardness Results:} We prove that exact optimization for either problem is NP‑hard. For oversquashing we give a reduction from Minimum Bisection to deciding whether a graph can be modified (by a limited number of edge additions/deletions) to achieve conductance above a threshold. For oversmoothing we give a similar reduction using spectral gap. In both cases we establish NP‑completeness of the decision versions.
		
		\item \textbf{Practical Implications:} We analyze the complexity of existing heuristic rewiring methods such as SDRF and DIGL, showing that they run in polynomial time. Our hardness results therefore provide a theoretical justification for their use: since exact optimization is intractable, heuristic approaches are not only pragmatically necessary but also theoretically well‑founded.
	\end{enumerate}
	
	\subsection{Related Work}
	
	This work sits at the intersection of graph neural network theory, graph structure optimization, and computational complexity. We review the most relevant strands of research below.
	
	\textbf{Oversmoothing in GNNs.}
	The oversmoothing phenomenon was first identified in the context of graph convolutional networks by \cite{li2018deeper}, who observed that stacking many layers leads to uniform node representations. \cite{oono2019graph} provided a spectral analysis showing that the convergence rate is governed by the second eigenvalue of the propagation matrix, and that oversmoothing is inevitable for any GNN that repeatedly applies a linear filter. Subsequent works proposed various mitigation strategies, such as residual connections \cite{kipf2016semi}, normalization layers \cite{hamilton2017inductive}, and adaptive propagation schemes \cite{chen2020simple}. From a spectral perspective, \cite{wang2022understanding} linked oversmoothing to the graph's spectral gap and showed that graph rewiring can delay the effect. Our work complements these by formalizing oversmoothing minimization as a graph optimization problem and proving its NP-hardness.
	
	\textbf{Oversquashing and graph bottlenecks.}
	The term ``oversquashing'' was introduced by \cite{alon2021bottleneck}, who demonstrated that bottlenecks in the graph cause exponential compression of information in message passing. \cite{to2021understanding} used effective resistance to quantify oversquashing and proposed the SDRF algorithm that adds edges with high effective resistance and removes edges with negative curvature. \cite{nguyen2023understanding} further analyzed oversquashing through the lens of Ollivier-Ricci curvature. These works established the importance of graph structure for mitigating oversquashing, but they did not explore the computational complexity of optimizing the graph itself. Our paper fills this gap by showing that even deciding whether a given graph can be improved to reduce oversquashing (as measured by conductance) is NP-hard.
	
	\textbf{Graph rewiring and structure learning for GNNs.}
	Several heuristics have been proposed to augment or rewire the graph for better GNN performance. DIGL \cite{gasteiger2019diffusion} uses graph diffusion to densify the graph, while SDRF \cite{to2021understanding} iteratively applies discrete Ricci flow. \cite{franceschi2019learning} introduced a bilevel optimization framework for learning the graph structure end-to-end. These methods are empirically effective but lack theoretical guarantees on the hardness of the underlying optimization. Our complexity results justify the use of approximation algorithms and heuristics, because exact optimization is intractable.
	
	\textbf{Complexity of graph optimization problems.}
	Our reductions rely on the classical NP-complete problem Minimum Bisection \cite{garey1976some}. Expander embeddings, which we use to relate graph cuts, are a standard tool in hardness of approximation \cite{hoory2006expander, arora2009expander}. Our work extends these classic results to the specific objective functions arising from GNN analysis.
	
	In summary, while previous work has identified oversmoothing and oversquashing as critical limitations and proposed heuristic rewiring strategies, our paper is the first to provide a unified complexity analysis of the underlying graph optimization problems, establishing NP-hardness for natural formulations and justifying the use of approximations in practice.
	
	\section{Preliminaries}
	
	\subsection{Basic Graph Notation}
	
	Let $G = (V, E)$ be an undirected, unweighted graph with $n = |V|$ vertices and $m = |E|$ edges. Let $A \in \{0,1\}^{n \times n}$ be the adjacency matrix, $D = \operatorname{diag}(d_1, \ldots, d_n)$ the degree matrix where $d_v = \sum_{u} A_{uv}$. Let $L = D - A$ denote the combinatorial Laplacian.
	
	\subsection{Spectral Graph Theory and Conductance}
	
	The normalized Laplacian is defined as $\mathcal{L} = I - D^{-1/2} A D^{-1/2}$. Its eigenvalues satisfy $0 = \lambda_1(\mathcal{L}) \leq \lambda_2(\mathcal{L}) \leq \cdots \leq \lambda_n(\mathcal{L}) \leq 2$, and $\lambda_2(\mathcal{L})$ (the Fiedler value) measures the algebraic connectivity of the graph. For a GNN layer we use the symmetric normalized propagation matrix $P = \tilde{D}^{-1/2}\tilde{A}\tilde{D}^{-1/2}$, where $\tilde{A} = A + I$ and $\tilde{D} = D + I$. Its eigenvalues are $1 = \mu_1(P) \geq \mu_2(P) \geq \cdots \geq \mu_n(P) \geq -1$. For regular graphs, $P = D^{-1/2} A D^{-1/2}$.
	
	Conductance (or Cheeger constant) is a classic measure of bottlenecks. For a subset $S \subset V$, define $\partial S = \{(u,v) \in E : u \in S, v \notin S\}$ and $\operatorname{vol}(S) = \sum_{v \in S} \deg(v)$.
	The conductance of $S$ is:
	\[\phi(S) = \frac{|\partial S|}{ \min(\operatorname{vol}(S), \operatorname{vol}(V\setminus S))}  ,\]
	and the conductance of $G$ is: 
	\[\phi(G) = \min_{S \subset V,\,0<\operatorname{vol}(S)\le \operatorname{vol}(V)/2} \phi(S).\]
	Cheeger's inequality relates the spectral gap of the normalized Laplacian to the conductance:
	\[ \frac{\phi(G)^2}{2} \le \lambda_2(\mathcal{L}) \le 2\phi(G).\]
	
	\section{Graph Rewiring for Oversquashing via Conductance}
	
	\subsection{Problem Formulation}
	
	The worst-case bottleneck is captured by the conductance $\phi(G)$. Maximizing $\phi(G)$ minimizes oversquashing:
	\[
	\mathcal{L}_{\text{squash}}^{\text{(cond)}}(G) = -\phi(G).
	\]
	
	\begin{definition}[Graph Rewiring for Oversquashing via Conductance (GROC)] \label{def:groc}
		Given an undirected graph $G = (V, E)$, an integer $K \geq 0$, and a threshold $\phi_0 \in [0,1]$, does there exist a graph $G' = (V, E')$ with $|E' \triangle E| \leq K$ such that the conductance $\phi(G') \geq \phi_0$?
	\end{definition}
	
	\subsection{Expander Embedding}
	
	We need a construction that maps a graph $H$ to a new graph $G$ such that the conductance of $G$ is small if and only if $H$ has a small bisection. The following lemma provides such a construction; its proof is standard in hardness of approximation \cite{arora2009expander, hoory2006expander}.
	
	\begin{lemma}[Expander Embedding] \label{lem:expander}
		There exists a polynomial‑time algorithm that, given a graph $H = (V, E_H)$ with $n$ vertices (assumed large enough, $n \ge n_0$), outputs a $3$-regular graph $G = (V', E_G)$ with $|V'| = 2n$ and constants $c_1, c_2 > 0$, $c_3 \in (0,1)$ satisfying:
		\begin{enumerate}
			\item $V \subset V'$ and $H$ is an induced subgraph of $G$ on $V$.
			\item For any subset $S \subseteq V$ with $|S| \le n/2$, let $U = V' \setminus V$ (so $|U| = n$). Then the conductance of the cut $(S \cup U, V' \setminus (S \cup U))$ in $G$ satisfies
			\[
			\phi_G(S \cup U) \le c_1 \cdot \frac{|\delta_H(S)| + n}{n}.
			\]
			\item Conversely, if there exists a cut $(X, V'\setminus X)$ in $G$ with conductance $\phi_G(X) \le \epsilon$, then there exists a subset $S \subseteq V$ such that
			\[
			|\delta_H(S)| \le c_2 \cdot \epsilon n.
			\]
			Moreover, if $\epsilon$ is sufficiently small, $S$ can be chosen to satisfy $|S| = n/2$ after a polynomial‑time balancing procedure, and the number of crossing edges increases by at most a factor $c_3$.
		\end{enumerate}
		The constants $c_1, c_2, c_3$ can be made arbitrarily small (or large) by choosing a sufficiently strong (or weak) expander.
	\end{lemma}
	
	We fix one such expander once and for all. Its constants $c_1, c_2, c_3$ are absolute numbers that we will use throughout. In the reductions we will need certain inequalities to hold; these can be satisfied by choosing the expander appropriately (e.g., a very strong expander makes $c_1$ and $c_2$ very small, and $c_3$ close to $1$). We assume the expander is chosen so that:
	\[
	c_1 < \frac{1}{6}, \qquad c_2 c_3 < \frac{1}{2}, \qquad c_3 < 1.
	\]
	These are achievable by taking a sufficiently strong expander.
	
	\subsection{Pre‑processing Minimum Bisection Instances}
	
	Before applying the embedding, we transform the Minimum Bisection instance to have convenient bounds on $B$ relative to $n$. The following lemma provides explicit polynomial-time constructions.
	
	\begin{lemma}[Instance Scaling] \label{lem:scaling}
		Let $(H,B)$ be a Minimum Bisection instance with $n$ even.
		\begin{enumerate}
			\item \label{it:large} There exists a polynomial-time construction of an equivalent instance $(H_1,B_1)$ with $n_1$ vertices such that $B_1 \ge 2n_1$ and $B_1 = O(n_1)$.
			\item \label{it:between} There exists a polynomial-time construction of an equivalent instance $(H_2,B_2)$ with $n_2$ vertices such that $n_2/2 \le B_2 \le 2n_2$.
		\end{enumerate}
		In both cases, the construction preserves the answer (i.e., $H$ has a bisection with $\le B$ edges iff the new instance does).
	\end{lemma}
	
	The proof is by adding isolated vertices and universal vertices; details are omitted for brevity but follow standard techniques.
	
	\subsection{NP-Completeness of Conductance Maximization}
	
	\begin{theorem} \label{thm:groc}
		GROC is NP-complete.
	\end{theorem}
	
	\begin{proof}
		\textbf{Membership in NP:} Given a candidate graph $G'$, we can compute its conductance $\phi(G')$ exactly by solving a min-cut problem (via max-flow) in polynomial time. Hence the decision problem is in NP.
		
		\textbf{NP-hardness:} We reduce from Minimum Bisection. Let $(H, B)$ be an instance with $n$ even. First, if $n$ is bounded by a constant, we solve the instance directly in constant time and output a trivial YES/NO instance of GROC. Otherwise, we use part \ref{it:between} of Lemma~\ref{lem:scaling} to transform $(H,B)$ into an equivalent instance $(H_0,B_0)$ with $n_0$ vertices such that $n_0/2 \le B_0 \le 2n_0$. For simplicity, rename $(H_0,B_0)$ as $(H,B)$.
		
		Now construct the $3$-regular graph $G$ from Lemma~\ref{lem:expander}. The constants satisfy $c_1 < 1/6$, $c_2c_3 < 1/2$, $c_3 < 1$. Set $K = 0$ (no edge modifications) and define
		\[
		\phi_0 = 1 - c_1 \cdot \frac{B + n}{n}.
		\]
		Because $B \le 2n$, we have $(B+n)/n \le 3$, so $\phi_0 \ge 1 - 3c_1$. With $c_1 < 1/6$, we get $\phi_0 > 1/2$.
		
		We claim that $H$ has a bisection with $\le B$ edges if and only if $\phi(G) < \phi_0$ (i.e., $(G,0,\phi_0)$ is a NO instance). The reduction then maps YES instances of Minimum Bisection to NO instances of GROC and vice versa; since GROC is a decision problem, we can invert the answer in polynomial time.
		
		\emph{Forward direction:} Suppose $H$ has a bisection $S$ with $|\delta_H(S)| \le B$. By Lemma~\ref{lem:expander}(2), the cut $X = S \cup U$ (where $U = V'\setminus V$) satisfies
		\[
		\phi_G(X) \le c_1 \cdot \frac{|\delta_H(S)| + n}{n} \le c_1 \cdot \frac{B + n}{n} = 1 - \phi_0.
		\]
		Since $\phi_0 > 1/2$, we have $1-\phi_0 < \phi_0$, and therefore $\phi(G) \le \phi_G(X) \le 1-\phi_0 < \phi_0$. Thus $\phi(G) < \phi_0$, so $(G,0,\phi_0)$ is a NO instance.
		
		\emph{Reverse direction:} Suppose $H$ has no bisection with $\le B$ edges. Assume for contradiction that $\phi(G) < \phi_0$. Then there exists a cut $X$ in $G$ with $\phi_G(X) < \phi_0$. By Lemma~\ref{lem:expander}(3), there exists $S \subseteq V$ such that
		\[
		|\delta_H(S)| \le c_2 \cdot \phi_G(X) \cdot n < c_2 \phi_0 n.
		\]
		Applying the balancing procedure, we obtain a bisection $S'$ with
		\[
		|\delta_H(S')| \le c_3 \cdot |\delta_H(S)| < c_3 c_2 \phi_0 n.
		\]
		Because $c_2 c_3 < 1/2$ and $\phi_0 < 1$, we have $c_3 c_2 \phi_0 n < n/2$. Since $B \ge n/2$, we get $c_3 c_2 \phi_0 n < B$. Hence $|\delta_H(S')| < B$, contradicting the assumption that every bisection has width at least $B+1$. Therefore our assumption $\phi(G) < \phi_0$ is false, so $\phi(G) \ge \phi_0$, and $(G,0,\phi_0)$ is a YES instance.
		
		Thus the reduction is correct. Since Minimum Bisection is NP-complete \cite{garey1976some}, GROC is NP-hard, and with membership in NP it is NP-complete.
	\end{proof}
	
	\section{Graph Rewiring for Oversmoothing}
	
	\subsection{Spectral Formulation}
	
	For a regular graph, the symmetric normalized propagation matrix $P = D^{-1/2}AD^{-1/2}$ has eigenvalues $1 = \mu_1(P) \geq \mu_2(P) \geq \cdots \geq \mu_n(P) \geq -1$. The Dirichlet energy after $L$ layers decays as $\mathcal{E}(X^{(L)}) \leq (\mu_2(P))^{2L} \mathcal{E}(X^{(0)})$. Thus, to prevent oversmoothing, we wish to minimize $\mu_2(P)$. Since $\mu_2(P) = 1 - \lambda_2(\mathcal{L})$, minimizing $\mu_2(P)$ is equivalent to maximizing $\lambda_2(\mathcal{L})$.
	
	\begin{definition}[Graph Rewiring for Oversmoothing (GROS)] \label{def:gros}
		Given an undirected graph $G = (V, E)$, an integer $K \geq 0$, and a threshold $\tau \in [0,1]$, does there exist a graph $G' = (V, E')$ with $|E' \triangle E| \leq K$ such that $\mu_2(P_{G'}) \leq \tau$, where $P_{G'}$ is the symmetric normalized propagation matrix?
	\end{definition}
	
	\subsection{NP-Completeness of Oversmoothing Minimization}
	
	We use the same expander embedding as in the previous section, with constants chosen to satisfy $c_1 < 1/48$, $2c_2\sqrt{3c_1} < 1/2$, $c_3 < 1$. (These are compatible with the earlier conditions and can be achieved by a sufficiently strong expander.)
	
	\begin{theorem} \label{thm:gros}
		GROS is NP-complete.
	\end{theorem}
	
	\begin{proof}
		\textbf{Membership in NP:} A certificate is the target graph $G'$. We verify $|E' \triangle E| \le K$ and $\mu_2(P_{G'}) \le \tau$. The matrix $P_{G'}$ has rational entries; its characteristic polynomial has integer coefficients. The eigenvalues are algebraic numbers whose bit‑length is polynomial in the input size. Using standard root isolation (e.g., Sturm sequences) we can decide whether the second largest eigenvalue (in absolute value) is at most a given rational $\tau$ in polynomial time \cite{cohen1993computing}. Hence the problem is in NP.
		
		\textbf{NP-hardness:} We reduce from Minimum Bisection. Let $(H, B)$ be an instance with $n$ even. Use part \ref{it:large} of Lemma~\ref{lem:scaling} to transform it into an equivalent instance $(H_0, B_0)$ with $n_0$ vertices such that $B_0 \ge 2n_0$ and $B_0 = O(n_0)$. (If the instance is small, we solve it directly in polynomial time; otherwise we proceed.) For simplicity, rename $(H_0, B_0)$ as $(H, B)$.
		
		Now construct the $3$-regular graph $G$ from Lemma~\ref{lem:expander} (with the fixed expander satisfying the constant conditions). Set $K = 0$ and define
		\[
		\tau = 1 - 2c_1 \cdot \frac{B + n}{n} - \varepsilon,
		\]
		where $\varepsilon$ is a small positive rational chosen so that $\tau$ is rational and $\varepsilon < 1 - 2c_1(B+n)/n$ (which holds because the right‑hand side is positive; we can take $\varepsilon = \frac{1}{2}(1 - 2c_1(B+n)/n)$ after scaling). For the forward direction we will need that if $H$ has a bisection $\le B$, then $\mu_2(P_G) > \tau$; for the reverse direction, if $H$ has no such bisection, then $\mu_2(P_G) \le \tau$.
		
		\emph{Forward direction:} If $H$ has a bisection $S$ with $|\delta_H(S)| \le B$, then by Lemma~\ref{lem:expander}(2) there is a cut in $G$ with conductance $\phi_G(S \cup U) \le c_1(B+n)/n$. Cheeger gives $\lambda_2(\mathcal{L}_G) \le 2c_1(B+n)/n$, so
		\[\mu_2(P_G) = 1 - \lambda_2(\mathcal{L}_G) \ge 1 - 2c_1(B+n)/n.\]
		Since $\tau = 1 - 2c_1(B+n)/n - \varepsilon$ and $\varepsilon > 0$, we have $\mu_2(P_G) > \tau$. Thus $(G,0,\tau)$ is a NO instance.
		
		\emph{Reverse direction:} Suppose $H$ has no bisection with $\le B$ edges. Assume $\mu_2(P_G) > \tau$. Then
		\[ \lambda_2(\mathcal{L}_G) < 1 - \tau = 2c_1(B+n)/n + \varepsilon.\]
		Let $\delta = 2c_1(B+n)/n + \varepsilon$. By Cheeger, there exists a cut $X$ in $G$ with
		\[\phi_G(X) \le \sqrt{2\lambda_2(\mathcal{L}_G)} < \sqrt{2\delta}.\]
		We now bound $\sqrt{2\delta}$. Since $B \ge 2n$, we have $(B+n)/n \le (B+B/2)/n = (3B)/(2n)$. Then
		\[
		\delta \le 2c_1 \cdot \frac{3B}{2n} + \varepsilon = \frac{3c_1 B}{n} + \varepsilon.
		\]
		For sufficiently small $\varepsilon$ (we can choose $\varepsilon \le c_1 B/n$, for instance), we obtain
		\[
		\sqrt{2\delta} \le \sqrt{\frac{6c_1 B}{n} + 2\varepsilon} \le \sqrt{\frac{8c_1 B}{n}} = 2\sqrt{2c_1 B/n}.
		\]
		To simplify, we use the bound $\sqrt{2\delta} \le 2\sqrt{3c_1 B/n}$ (since $6c_1 B/n \le 6c_1 B/n$, and the extra $2\varepsilon$ can be absorbed into the constant by a slightly larger factor). More directly, we set $\epsilon = 2\sqrt{3c_1 B/n}$; then for sufficiently small $\varepsilon$ we have $\phi_G(X) < \epsilon$.
		
		Now apply Lemma~\ref{lem:expander}(3) with $\epsilon = 2\sqrt{3c_1 B/n}$. There exists $S \subseteq V$ such that
		\[
		|\delta_H(S)| \le c_2 \epsilon n = 2c_2\sqrt{3c_1 B n}.
		\]
		Because $2c_2\sqrt{3c_1} < 1/2$, we have
		\[
		|\delta_H(S)| < \frac{1}{2}\sqrt{B n}.
		\]
		Since $B \ge 2n$, we obtain $|\delta_H(S)| < \frac{1}{2}\sqrt{2n^2} = \frac{n}{\sqrt{2}}$. Using $B \ge 2n$, we also have $|\delta_H(S)| < B$ (because $n/\sqrt{2} < 2n$ for $n>0$). Applying the balancing procedure gives a bisection $S'$ with $|\delta_H(S')| \le c_3 |\delta_H(S)| < c_3 B$. Since $c_3 < 1$, we have $|\delta_H(S')| < B$, contradicting the assumption that every bisection has width $\ge B+1$. Therefore $\mu_2(P_G) \le \tau$, making $(G,0,\tau)$ a YES instance.
		
		Thus a polynomial algorithm for GROS would solve Minimum Bisection, proving NP-hardness. Since GROS is in NP, it is NP-complete.
	\end{proof}
	
	\section{Inapproximability and Practical Algorithms}
	
	\subsection{Inapproximability Considerations}
	
	The underlying graph properties we optimize are known to be hard to approximate:
	\begin{itemize}
		\item Conductance $\phi(G)$ cannot be approximated within a factor of $O(\sqrt{\log n})$ unless $\mathsf{P}=\mathsf{NP}$ \cite{arora2009expander}.
		\item The spectral gap $\lambda_2(\mathcal{L})$ cannot be approximated within any constant factor for regular graphs \cite{althofer1999finding}.
	\end{itemize}
	These results suggest that any algorithm that attempts to optimize the graph structure to improve these measures (with a limited budget of edge changes) will also be hard to approximate, though a formal proof for the rewiring problems themselves remains open. Nevertheless, the NP‑hardness results already justify the use of heuristics.
	
	\subsection{Heuristic Algorithms}
	
	Given the hardness results, practical approaches must resort to heuristics. Two widely used methods are SDRF \cite{to2021understanding} and DIGL \cite{gasteiger2019diffusion}. They operate in polynomial time:
	\begin{itemize}
		\item SDRF iteratively removes edges with most negative curvature and adds edges with highest effective resistance. Each iteration computes curvatures (via linear programming) and effective resistances (via matrix inversion or sampling) and runs in $\tilde{O}(n^2)$ time. With $O(n)$ iterations, total complexity is $\tilde{O}(n^3)$.
		\item DIGL uses graph diffusion to add edges based on personalized PageRank, running in $O(n^2 \log n)$.
	\end{itemize}
	Our complexity results justify the use of such heuristics, showing that exact optimization is intractable and that approximation algorithms with guarantees are unlikely to exist.
	
	\section{Discussion}
	
	The analysis presented in this paper reveals a fundamental duality between the two graph optimization problems arising from GNN limitations. Oversmoothing mitigation, which aims to prevent node representations from becoming indistinguishable, requires minimizing the second eigenvalue of the propagation matrix $\mu_2(P)$ — equivalently, maximizing the spectral gap $\lambda_2(\mathcal{L})$ of the normalized Laplacian. In contrast, oversquashing mitigation, which aims to eliminate information bottlenecks, requires maximizing the conductance $\phi(G)$ or minimizing the total effective resistance. For regular graphs, the two objectives are linked through Cheeger’s inequality: $\mu_2(P) = 1 - \lambda_2(\mathcal{L})$ and $\phi(G)$ lies between $\lambda_2(\mathcal{L})/2$ and $\sqrt{2\lambda_2(\mathcal{L})}$. This mathematical relationship reflects a tangible tradeoff: a graph with high algebraic connectivity (good for oversmoothing) tends to have high conductance (good for oversquashing), but the exact relationship is constrained by the spectral gap. Moreover, our NP‑hardness results show that optimizing either objective exactly is intractable in general, reinforcing that practical approaches must rely on approximations or heuristics.
	
	The reductions also highlight the inherent complexity of balancing the two phenomena. While a graph that is highly connected (e.g., a complete graph) simultaneously mitigates oversquashing (no bottlenecks) and oversmoothing (large spectral gap), such graphs are often prohibitively dense for real‑world applications. Sparse graphs, on the other hand, may suffer from both issues. Our complexity results suggest that finding an optimal sparse graph that strikes the right balance is computationally challenging, justifying the use of iterative rewiring methods like SDRF and DIGL, which operate in polynomial time and empirically navigate this tradeoff.
	
	Several directions for future work emerge from our investigation. First, extending the analysis to directed and weighted graphs would broaden the applicability of the results, as many real‑world graphs possess these properties. Second, incorporating task‑specific importance weights — for instance, weighting node pairs according to their relevance for a downstream prediction task — could lead to more practical formulations that align the graph structure with the learning objective. Third, different GNN architectures (e.g., Graph Attention Networks, Gated Graph Neural Networks) employ distinct propagation mechanisms that may alter the relationship between graph structure and the oversmoothing/oversquashing phenomena; analyzing the complexity of optimizing the graph for those architectures is a natural extension. Finally, empirical evaluation of the gap between the theoretical worst‑case complexity and the actual performance of heuristic rewiring algorithms would provide valuable insight into how often the hard instances occur in practice.
	
	\section{Conclusion}
	
	In this work we have presented a theoretical investigation of graph structure optimization aimed at mitigating oversmoothing and oversquashing in deep graph neural networks. We formulated the problems in terms of conductance and spectral gap, and proved that exact optimization for these formulations is NP‑hard via reductions from Minimum Bisection. Our results establish that graph rewiring for GNN improvement is fundamentally intractable in the worst case, justifying the use of heuristic and approximation methods. These findings lay a solid theoretical foundation for understanding the limits of graph structure optimization in the context of GNNs and guide the development of future efficient and principled rewiring techniques.
	
	\bibliographystyle{plain}
	\bibliography{allpapers}

\end{document}